\documentclass[]{style}

\usepackage{amssymb}
\usepackage{xcolor}
\usepackage{tcolorbox}
\usepackage[utf8]{inputenc}
\usepackage[T1]{fontenc}
\usepackage{lmodern}
\usepackage{tcolorbox}
\usepackage{array}
\usepackage{tikz}
\usepackage{xcolor}

\definecolor{txcolor}{RGB}{27, 94, 32}   
\definecolor{rxpim}{RGB}{24, 95, 165}
\definecolor{outpim}{RGB}{127, 119, 221}
\definecolor{rxshade}{RGB}{59, 139, 212}
\definecolor{txshade}{RGB}{27, 94, 32} 

\usepackage{array}
\usepackage{makecell}
\usepackage{soul}

\usepackage{graphicx}
\usepackage{tcolorbox}
\usepackage{fontawesome5}
\usepackage{xcolor}
\usepackage{booktabs}
\usepackage{multirow}
\usepackage{algpseudocode}
\usepackage{listings}
\usepackage{enumitem}
\usepackage{caption}
\usepackage{mdframed}
\usepackage{tcolorbox}
\usepackage{algorithm}
\usepackage{algpseudocode}  
\usepackage{amsmath}        
\usepackage{tabularx}
\usepackage{algpseudocode}
\usepackage{hyperref}
\usepackage{amsthm}
\usepackage{orcidlink}

\usepackage{amsmath,amsfonts,bm}

\def\eqref#1{equation~\ref{#1}}

\def\1{\bm{1}}

\DeclareMathAlphabet{\mathsfit}{\encodingdefault}{\sfdefault}{m}{sl}
\SetMathAlphabet{\mathsfit}{bold}{\encodingdefault}{\sfdefault}{bx}{n}

\newtcolorbox{takeawaybox}[1]{
    colback=green!5!white,
    colframe=green!50!black,
    fonttitle=\bfseries,
    title=Takeaway
}

\newtheorem{lemma}{Lemma}
\newtheorem{corollary}{Corollary}
\newcommand{\unc}[1]{\textcolor{gray}{\scriptsize$\pm\,#1$}}

\title{ERank in Latent Space as an Image-Complexity and Richness Measure} 

\author[]{Maksim Smirnov}
\author[]{Grigory Kononov}
\author[]{Anastasiia Linich}
\author[]{Egor Surkov}
\author[]{Egor Shvetsov}

\abstract{
We propose the effective rank (ERank) of the channel covariance of an
image's deep feature map as a per-sample, label-free measure of visual
richness, computed from a single forward pass through a frozen
pretrained encoder. ERank counts how many decorrelated channel
directions an image activates, and we characterize its properties,
including its behavior under noise. Empirically, ERank orders images from plain to
visually rich,
correlates with codec bitrate, sharpness, and edge density, and correlates with
human complexity annotations on IC9600 with $r = 0.72$. As a
data-selection criterion, removing low-ERank samples improves
super-resolution and removing high-ERank samples improves OCR, in both
pretraining and finetuning, while selection does not help
classification, segmentation, or denoising. ERank is thus a cheap
richness signal, useful exactly when task difficulty is governed by
input richness.
\date{\today}
}

\begin{document}

\maketitle

\newtheorem{proposition}{Proposition}
\section{Introduction}
In this work we demonstrate that the \emph{effective rank}
(ERank)~\citep{roy2007effective} computed over channels or tokens of an
image's intermediate representation corresponds to image richness and
diversity. To compute ERank, an image is passed through a pretrained
model, its spatial feature map from a given layer is flattened to
obtain an $HW \times C$ matrix, and the exponential of the Shannon
entropy of the normalized eigenvalues of the channel covariance is
computed. The quantity has an appealing interpretation as it counts
how many decorrelated channel directions the image activates, so a
higher ERank corresponds to a less correlated, more diverse feature
map. Beyond this intuition, the measure has appealing properties. It is a
smooth, differentiable relaxation of matrix rank, invariant to
isotropic scaling and to rotations of channel space; we discuss these
properties in detail and provide connections to other measures in
Appendix~\ref{sec:properties}.

We validate the measure along four axes. \textbf{First}, \emph{visual
analysis} shows that sorting images by ERank produces an interpretable
ordering from plain to visually rich images. \textbf{Second}, our analysis demonstrates that ERank
correlates with codec file size, sharpness, and edge density
($r \approx 0.5$) but not with distortion or colorfulness
($r \approx 0.03$). Moreover, it strongly correlates with human
complexity annotations on IC9600~\citep{feng2023ic9600} with
$r = 0.72$. \textbf{Third}, we evaluate ERank as a \textit{data-selection criterion}, where
training samples are ranked by their ERank score, a fraction is
pruned, and the task model is trained on the rest. We evaluate both
pretraining from scratch and finetuning from a checkpoint. Removing
low-ERank samples improves super-resolution on DIV2K, while removing
high-ERank samples helps OCR on IIIT5K, where visually rich samples
are predominantly hard or noisy. In both tasks ERank beats random
removal across pruning budgets. Finally, we demonstrate where data selection with ERank does not help. ERank-based selection does not improve
classification, segmentation, or denoising, and we clarify why in the
corresponding sections.

\section{Related Work}
\label{sec:related}
 
Label-free evaluation of per-sample characteristics is important in a
wide range of applied tasks, including curriculum and active 
learning~\citep{toneva2019forgetting, sener2018coreset} and data selection for pretraining
at scale~\citep{sorscher2022beyond,abbas2023semdedup}. Moreover, per-sample complexity scores have been used to allocate
computation in dynamic architectures to spend less compute on easy
inputs~\citep{han2021dynamic,huang2018msdnet,wang2021dvt}, in particular
ICAR~\citep{williamslekuona2026imagecomplexityawareadaptiveretrieval} routes images through a
vision-language encoder at a depth chosen by a predicted complexity
score.

\textbf{Quality metrics in pixel space.}
A long line of proxies operates directly in pixel space. Shannon
entropy of the intensity histogram~\citep{shannon1948mathematical} and
colorfulness indices~\citep{hasler2003colorfulness} summarize
first-order statistics; fractal dimension captures the self-similarity
of natural textures~\citep{pentland1984fractal}; edge density derived
from the Canny detector~\citep{canny1986computational} and
gradient-energy sharpness scores such as
Tenengrad~\citep{krotkov1988focusing,pertuz2013analysis} quantify
high-frequency content; and the file size produced by a standard codec
such as JPEG~\citep{wallace1992jpeg} serves as a computable surrogate
for Kolmogorov complexity~\citep{li2008kolmogorov} and as a difficulty
signal. No-reference quality measures such as
BRISQUE~\citep{mittal2012brisque} score distortion and naturalness
rather than richness. These measures are cheap but blind to semantics
and confounded by pixel noise. To move beyond pixel statistics, the
unsupervised activation energy (UAE)~\citep{SARAEE2020102949} measures
image complexity in the activation space of a pretrained encoder,
similar in spirit to our approach. We evaluate its correlation with
human complexity annotations alongside ERank in
Section~\ref{sec:compression}.

\textbf{Metrics for Data Selection.}
Data pruning and coreset selection typically rely on one of two
strategies - removing easy and keeping hard samples, as ranked by
training dynamics or error
signals~\citep{toneva2019forgetting,paul2021el2n,xia2023moderate}, or
selecting samples that better cover the data
distribution~\citep{sener2018coreset,abbas2023semdedup}.   In this work we demonstrate
that per-sample diversity is a complementary
signal which enables effective selection of training subsets for
super-resolution and OCR.
 
\textbf{Metrics for Dataset Characterization.}
Characterizing data collectively, at the level of a full dataset in a
latent space, is well established for measuring model quality and
sample-set diversity. FID~\citep{heusel2017fid}, Inception
Score~\citep{salimans2016is},
precision/recall~\citep{sajjadi2018precrec,kynkaanniemi2019improved},
and density/coverage~\citep{naeem2020reliable} measure fidelity and
coverage of a sample set against a reference distribution, treating
each sample as a single embedding vector. ERank has likewise been
applied at this level, the Vendi Score~\citep{friedman2023vendi}
computes ERank of a kernel matrix of pairwise sample similarities to
measure the diversity of a collection, and
RankMe~\citep{garrido2023rankme} computes ERank of the covariance of
sample embeddings over a dataset to evaluate self-supervised models
without labels. 
While these works are the closest to ours, the key difference is that
we compute ERank \emph{within} a single sample, over the channels or
tokens of its feature map, yielding a per-sample score rather than a
set-level one.

\section{ERank as a Measure of Image Richness}
\label{sec:method}

\subsection{Setup and Definition}
\label{sec:setup}
Let $f$ be a pretrained encoder and let $f^{(\ell)}$ denote its
intermediate representation at layer $\ell$. For an image $I$, the
feature map $f^{(\ell)}(I)$ is flattened into
$X^{(\ell)} \in \mathbb{R}^{N \times C}$, where $N = HW$ is spatial size and $C$ the
number of channels (or tokens, for transformer encoders) at layer $\ell$. Let $\bar{X}^{(\ell)}$ denote
$X^{(\ell)}$ after centering each column (channel), and let
\begin{equation}
  \Sigma^{(\ell)} \;=\; \tfrac{1}{N}\,\bar{X}^{(\ell)\top}\bar{X}^{(\ell)}
  \;\in\; \mathbb{R}^{C \times C}
\end{equation}
be the empirical channel covariance at layer $\ell$, with eigenvalues
$\lambda_1 \ge \dots \ge \lambda_C \ge 0$. Writing
$p_i = \lambda_i / \sum_{j} \lambda_j$ for the normalized spectrum, the
\emph{ERank}~\citep{roy2007effective} of the image at layer $\ell$ is
\begin{equation}
  \label{eq:erank}
  \operatorname{erank}(I; f, \ell)
  \;=\;
  \exp\!\Big( -\textstyle\sum_{i=1}^{C} p_i \log p_i \Big),
\end{equation}
with the convention $0 \log 0 = 0$. Richness is thus always measured
\emph{relative to a chosen encoder $f$ and layer $\ell$}.
$\operatorname{erank}(I)$ We write erank(I) for brevity.

\textbf{Properties of ERank.} We summarize the properties of ERank, which determine where it applies and where it fails, deferring statements and proofs to Appendix~\ref{sec:properties}. ERank is bounded between $1$  and $\min(N,C)$, invariant
to isotropic scaling and to rotations of channel space, and unlike
matrix rank continuous and differentiable. An exact identity ties
it to pairwise channel correlation and to mean pairwise similarity via
the Rényi family - ERank is
maximal when channels are decorrelated and collapses as they correlate.
It increases monotonically under additive decorrelated noise, so it cannot separate detail from noise on
its own - a limitation resolved only by the encoder, motivating the
next section.

\textbf{Why Measure Richness in Feature Space?} Three considerations
motivate computing ERank on features rather than pixels. \emph{(i)
Dimensionality:} the pixel channel spectrum is three-dimensional (RGB),
capping ERank at~3, whereas encoders provide from dozens to thousands of
semantically tuned channels. \emph{(ii) Manifold structure:} natural
images concentrate near low-dimensional structure, and pretrained
encoders parameterize coordinates along
it~\citep{ansuini2019intrinsic,pope2021intrinsic}, so channel activity
reflects position on the image manifold rather than raw pixel
statistics. \emph{(iii) Implicit low-rank bias:} gradient-based training is
biased toward low-rank solutions~\citep{arora2019implicit,gunasekar2017implicit,huh2021lowrank},
and trained representations exhibit rapidly decaying
eigenspectra~\citep{jing2022understanding,papyan2020neural} - a trained
encoder therefore does not maintain channel directions it does not
need for its objective. When an image nevertheless elicits a broad channel spectrum, that
breadth is informative about the \emph{input} rather than an artifact of
the representation.

\section{Experiments \& Results}
\label{sec:experiments}

\subsection{Experimental Setup}
\label{sec:setup-exp}

\textbf{Scoring backbones.}
Every image is scored by a single forward pass through one of two pretrained encoders: \textbf{ResNet-18}~\citep{he2016resnet}
pretrained on ImageNet and \textbf{CLIP ViT-B/32}~\citep{radford2021clip}.
We use layers $\{1,2,3,4\}$ of the ResNet and blocks $\{2,5,8,11\}$ of
CLIP, each feature map is flattened into a matrix of spatial positions
(patch tokens for CLIP) by channels. ERank is computed on each layer
separately and averaged, yielding one  per-image score. The layers were fixed a priori to span the network depth and
were not tuned.

\textbf{Datasets and task models.}
The correlation analysis 
use random LAION~\citep{schuhmann2022laion} subsets of fixed sizes  and full
IC9600~\citep{feng2023ic9600}. The
data-selection experiments (Sec.~\ref{sec:helped}) cover five tasks:
\textbf{(1)} $\times4$ super-resolution (DIV2K~\citep{agustsson2017div2k},
 EDSR~\citep{lim2017edsr}), \textbf{(2)} scene-text recognition
(IIIT5K~\citep{mishra2012iiit5k}, CRNN~\citep{shi2017crnn} with CTC
loss~\citep{graves2006ctc}), \textbf{(3)} classification
(MNIST~\citep{lecun1998mnist}, CIFAR-10~\citep{krizhevsky2009cifar},
ResNet-18), \textbf{(4)} Gaussian denoising (DIV2K, $\sigma{=}25$, EDSR), and
\textbf{(5)} semantic segmentation (Pascal VOC~\citep{everingham2010pascal},
DeepLabV3--ResNet50~\citep{chen2017deeplabv3}). Selection is evaluated
in two regimes  \emph{pretrain}, training from scratch and \emph{finetune}. For finetuning we used EDSR $\times4$ from an EDSR $\times2$
checkpoint trained on all $800$ DIV2K images, and CRNN with a frozen
ImageNet ResNet-18 stem and trainable CTC head.

\textbf{Pruning protocol.}
We rank the training set by a score, remove $10$--$50\%$ of the
lowest or highest-scoring samples, retrain on the remainder, and
compare against \textbf{random} removal at the same budget as a  data-selection
baseline~\citep{guo2022deepcore,sorscher2022beyond}. All models are
trained to convergence.

\subsection{What ERank selects - visual analysis}
\label{sec:visual}

In Figure~\ref{fig:extremes} we demonstrate the highest- and lowest-scoring samples of the
ERank across the datasets, scores are obtained with CLIP model as encoder. The pattern is consistent across datasets, 
\textit{Low-ERank} images are plain and low-detail, whereas \textit{High-ERank} images
are visually busy, textured, multi-object, and high-frequency.
\begin{figure*}[t]
  \centering
  \includegraphics[width=1\linewidth]{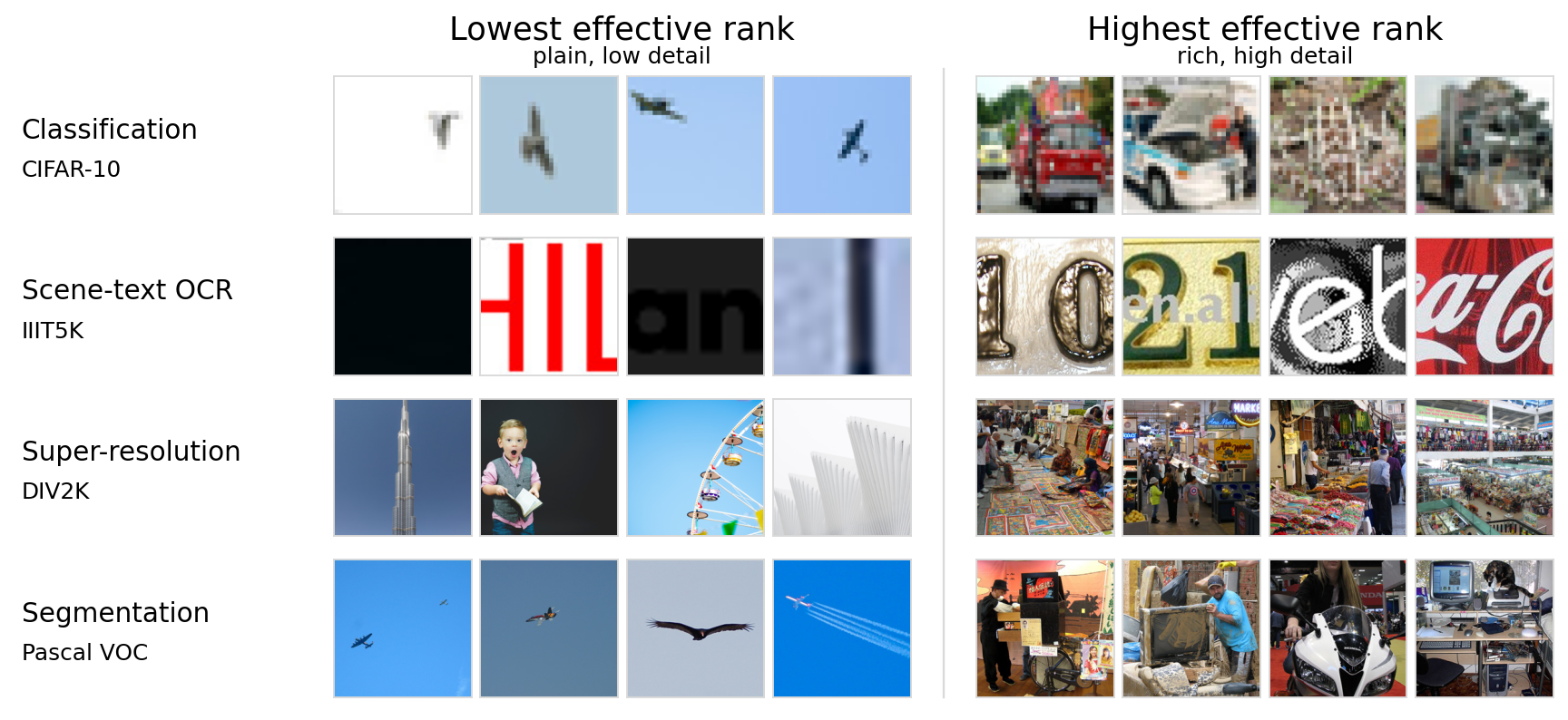}
  \caption{Images with highest and lowest ERank obtained with CLIP encoder across five task specific datasets.}
  \label{fig:extremes}
\end{figure*}

\subsection{ERank is an image-complexity score}
\label{sec:compression}

Visual analysis suggests that ERank is an image-complexity measure. We
test this directly by correlating ERank with five representative
statistics on LAION (Table~\ref{tab:whatis}): (1) Tenengrad sharpness, (2) JPEG
bytes-per-pixel as a compressibility proxy, (3) Canny edge density, (4) No-Reference IQA BRISQUE
distortion, and (5) colorfulness. The largest associations are with compressibility, sharpness, and edge density, while distortion and colorfulness are negligible. This supports interpreting ERank as a compact detail/complexity signal rather than as a generic quality or color score.

\textbf{Controlled corruption probe.} We further probe how ERank reacts to
degradations of a single image for that we corrupt 1000 LAION images with cutout,
Gaussian pixel noise, and Gaussian blur, and report rank-shift of
each image relative to its clean version, results are presented in Table~\ref{tab:corruption}. Blur
causes the largest drops at early and middle layers of both backbones,
confirming that ERank tracks high-frequency detail, while cutout shifts ranks
only mildly. The sign flips at the deepest layers motivate averaging ERank across depths
(or using middle layers only). We attribute the effect to late layers
encoding task-aligned directions where a confidently recognized image
concentrates on few such directions, while a corrupted, ambiguous input
spreads energy across many, raising ERank.
\begin{figure*}[ht]
  \centering
  \begin{minipage}[t]{0.55\textwidth}
    \vspace{0pt}
    \centering
    \captionof{table}{Controlled corruption probe on LAION (1000 samples).
    Each cell is the paired ERank rank-shift (corrupted minus clean).}
    \label{tab:corruption}
    {\renewcommand{\arraystretch}{1.3}%
    \resizebox{\linewidth}{!}{%
    \begin{tabular}{@{}llr@{\,}lr@{\,}lr@{\,}l@{}}
      \toprule
      Model & Layers
        & \multicolumn{2}{c}{Cutout}
        & \multicolumn{2}{c}{Noise}
        & \multicolumn{2}{c@{}}{Blur} \\
      \midrule
      \multirow{3}{*}{ResNet-18}
        & 1  & $-36$  & \unc{49}  & $-158$ & \unc{124} & $-498$ & \unc{288} \\
        & 3 & $+4$   & \unc{79}  & $-127$ & \unc{127} & $-322$ & \unc{181} \\
        & 4   & $+101$ & \unc{150} & $-136$ & \unc{211} & $+166$ & \unc{272} \\
      \midrule
      \multirow{3}{*}{CLIP ViT-B/32}
        & 2  & $-239$ & \unc{164} & $+45$  & \unc{148} & $-490$ & \unc{281} \\
        & 6 & $-71$  & \unc{106} & $-42$  & \unc{105} & $-409$ & \unc{227} \\
        & 11   & $+32$  & \unc{93}  & $-103$ & \unc{128} & $-368$ & \unc{237} \\
      \bottomrule
    \end{tabular}%
    }}
  \end{minipage}\hfill
  \begin{minipage}[t]{0.42\textwidth}
    \vspace{0pt}
    \centering
    \captionof{table}{ERank correlation (PLCC) with qualitative metrics on
    LAION (2000  samples). ERank tracks detail, not colour or distortion.}
    \label{tab:whatis}
    \small
    \begin{tabular}{@{}lcc@{}}
      \toprule
      Metric & $|r|$ & $p$-value \\
      \midrule
      Sharpness (Tenengrad)   & $\mathbf{0.47}$ & $<0.05$ \\
      Compressibility (JPEG)  & $\mathbf{0.52}$ & $<0.05$ \\
      Edge density            & $\mathbf{0.49}$ & $<0.05$ \\
      Distortions (BRISQUE)   & $0.02$ & $0.28$ \\
      Colourfulness           & $0.03$ & $0.27$ \\
      \bottomrule
    \end{tabular}
  \end{minipage}
\end{figure*}

\textbf{Validation against human complexity labels.} Further, we measure the Pearson correlation of ERank with \emph{human} judgments on IC9600~\citep{feng2023ic9600}, a benchmark of $9{,}600$ images each annotated with an image-complexity score. Alongside ResNet ERank, we evaluate \emph{Unsupervised Activation Energy} (UAE) on
IC9600 as the strongest unsupervised image complexity proxy according to ~\citep{SARAEE2020102949}. UAE is the mean nonnegative activation
over spatial positions and channels, averaged over the same four layers.
The results in Figure~\ref{fig:ic9600} demonstrate that ResNet ERank strongly tracks the human label -  $r = 0.72$, while the corresponding ResNet UAE score is weaker - $r = 0.53$ (or $r = 0.68$ for the original implementation using VGG-16).

\begin{figure*}[ht!]
  \centering
  \includegraphics[width=0.7\textwidth]{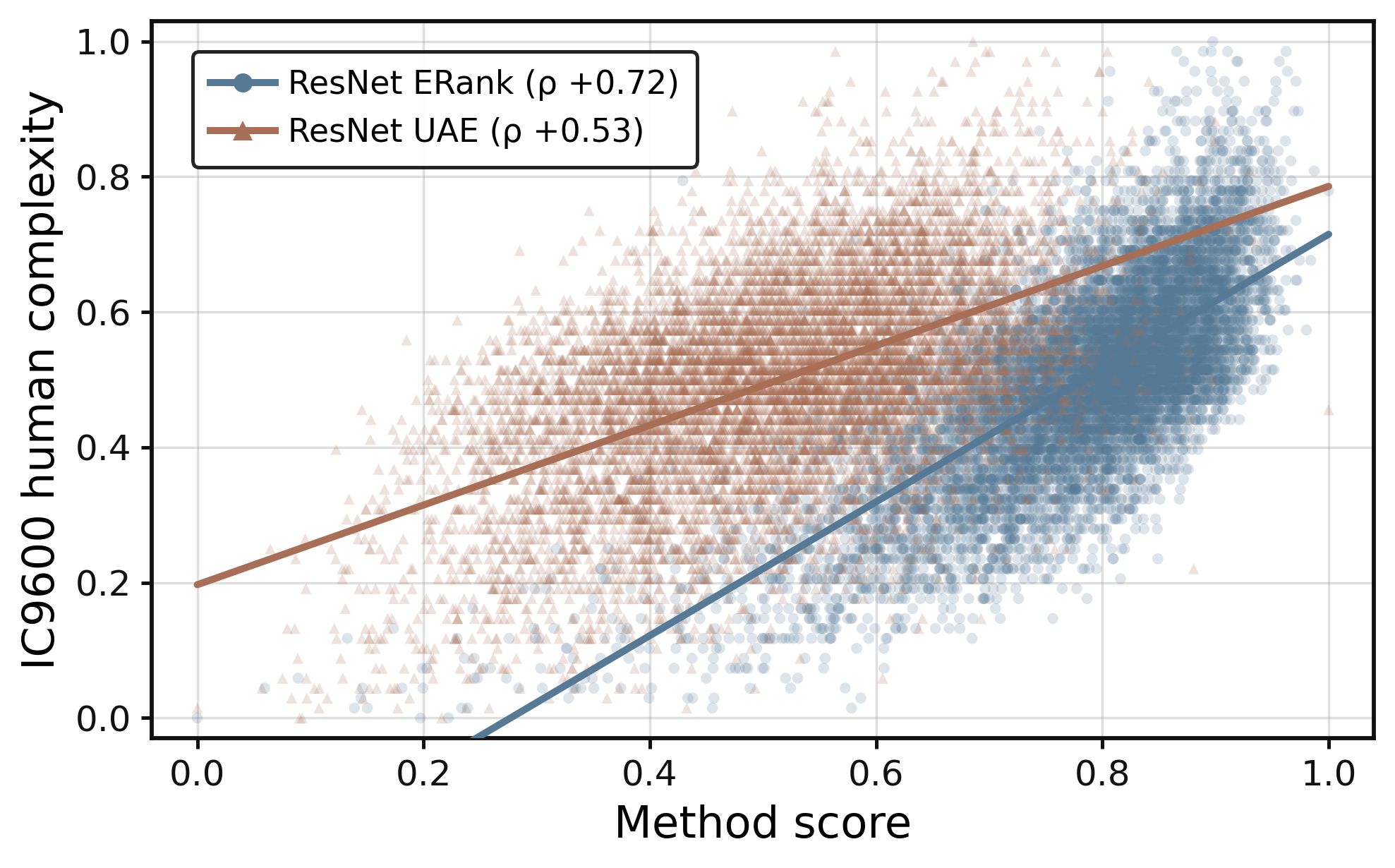}
  \caption{Correlation (PLCC) of ERank and UAE~\citep{SARAEE2020102949} with human annotated image-complexity scores on
  IC9600~\citep{feng2023ic9600}. Both $r$ are statistically significant with $p$ value below $0.05$.}
  \label{fig:ic9600}
\end{figure*}

\subsection{Data Selection with ERank}
\label{sec:helped}
We further analyse whether ERank-based data selection outperforms random
sampling for model pretraining and finetuning. Random sampling is a
notoriously strong baseline for data
selection~\citep{guo2022deepcore, okanovic2024repeated, ayed2023data, moser2025coreset},
often matching or exceeding sophisticated score-based methods, particularly
at moderate pruning ratios~\citep{sorscher2022beyond}. We therefore compare
ERank-based selection directly against random sampling under matched budgets
on various tasks.

\textbf{Where ERank helps.}
\textit{Super-resolution (DIV2K $\times4$).} Removing plain, low-ERank
images whose flat patches carry little high-frequency content to
reconstruct improves PSNR over random removal at every budget
(Figures~\ref{fig:sr-retention-pretrain} and~\ref{fig:sr-retention-finetune},
Table~\ref{tab:helped}). Intuitively, the SR model learns little from plain
images.
\textit{Scene-text recognition (IIIT5K).} Here the pattern reverses, the
visually \emph{rich} images are the \emph{hard} ones, as dense texture and
cluttered backgrounds hurt recognition, so removing the highest-ERank
examples helps (Table~\ref{tab:helped}). Overall, CLIP-derived ERank scores
outperforms or matches random at most budgets, with the exception of the 30\% pretraining budget on OCR, whereas
ResNet-derived scores underperform on OCR in both scenarios.

\begin{figure*}[t]
  \centering
  \begin{subfigure}[t]{0.48\textwidth}
    \centering
    \includegraphics[width=\linewidth]{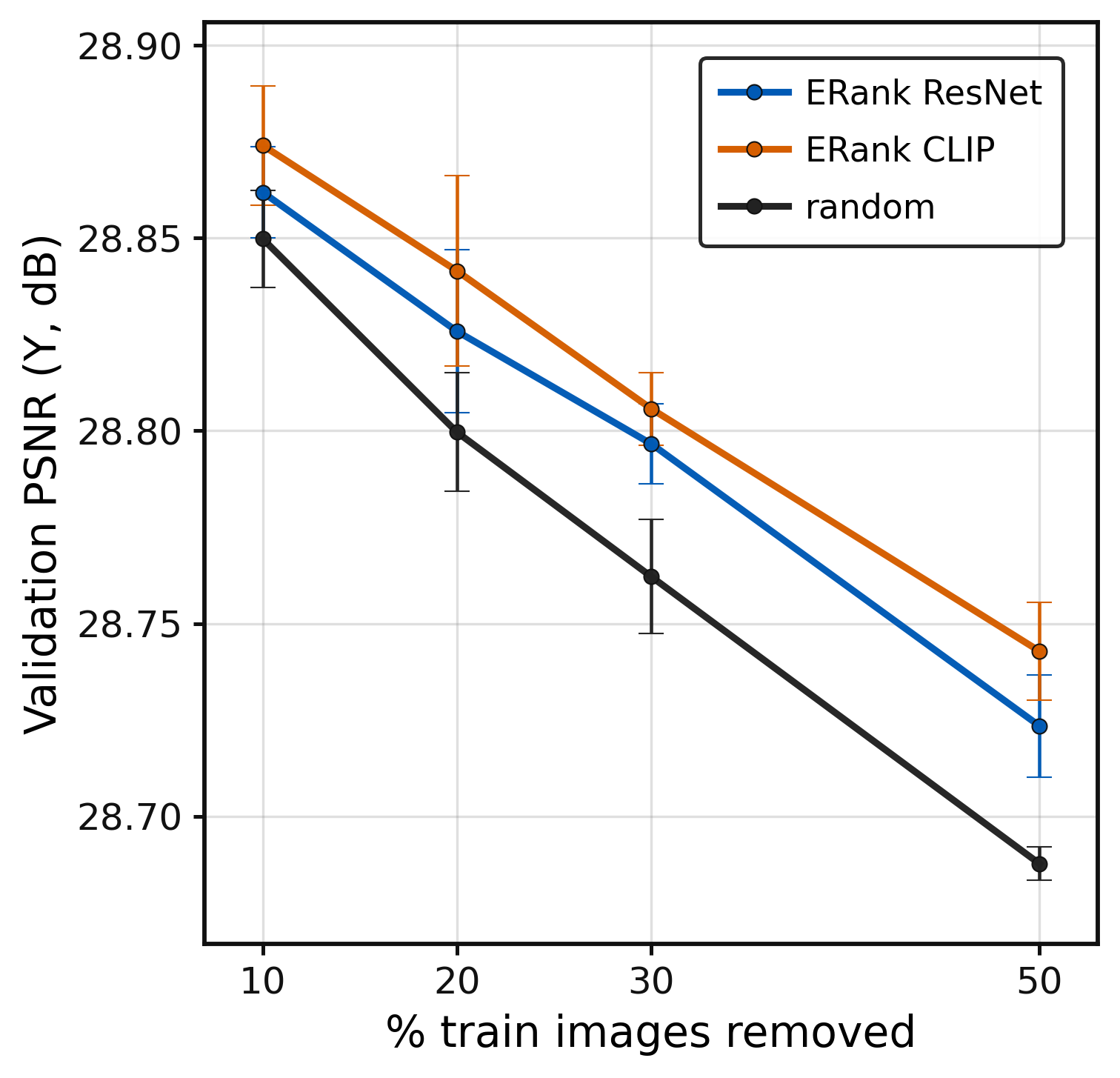}
    \caption{Training from scratch}
    \label{fig:sr-retention-pretrain}
  \end{subfigure}\hfill
  \begin{subfigure}[t]{0.48\textwidth}
    \centering
    \includegraphics[width=\linewidth]{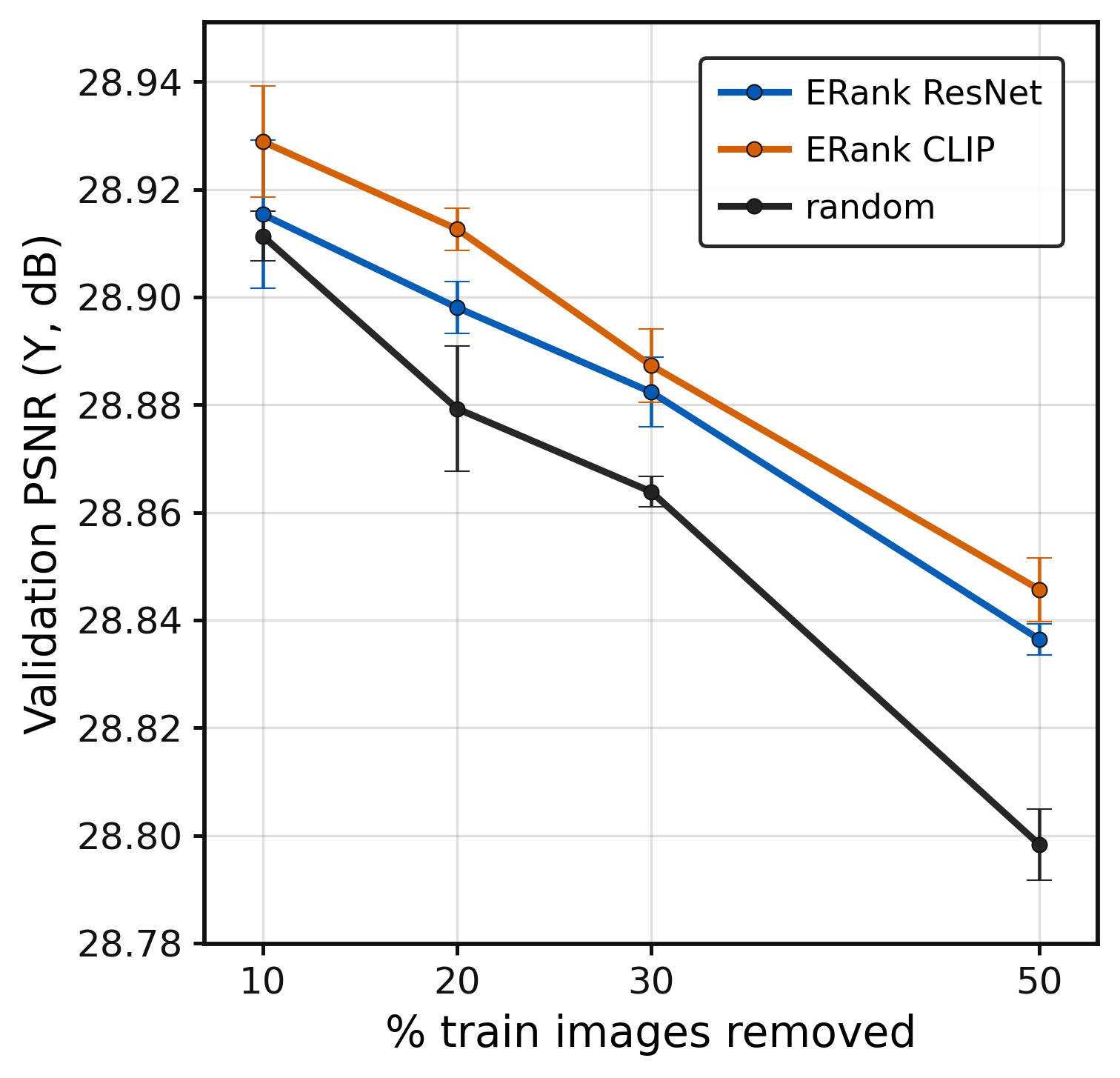}
    \caption{Finetuning from $\times2$ checkpoint}
    \label{fig:sr-retention-finetune}
  \end{subfigure}
  \caption{DIV2K $\times4$ SR retention. The curves show ERank with the
  ResNet-18 and CLIP ViT-B/32 scoring backbones, together with one shared
  random-removal baseline. Curves are distinguished only by color; error bars
  show the standard deviation over three task-model seeds.}
  \label{fig:sr-retention}
\end{figure*}

\newcommand{\std}[1]{\textcolor{gray}{\scriptsize$\pm$#1}}

\begin{table*}[ht!]
  \centering
  \caption{\textbf{Data Selection results} on DIV2K $\times4$ super-resolution  and IIIT5K OCR. Results are
  reported for ERank computed
  from ResNet-18 or CLIP ViT-B/32 features. Cells are mean\std{std} over
  three seeds. Best performance reported in bold.}
  \label{tab:helped}
  \small
  \begin{tabular}{@{}llcccc@{}}
    \toprule
    & & \multicolumn{4}{c@{}}{Pruning ratio} \\
    \cmidrule(l){3-6}
    & Selection & 10\% & 20\% & 30\% & 50\% \\
    \midrule
    \multicolumn{6}{@{}l}{\emph{DIV2K $\times4$ super-resolution, PSNR (dB) $\uparrow$ } - \textbf{remove low-ERank images}}\\
    \midrule
    \multirow{3}{*}{\textbf{Pretrain}}
      & Random               & 28.850\std{.013} & 28.800\std{.015} & 28.762\std{.015} & 28.688\std{.004} \\
      & ERank (ResNet-18)    & 28.862\std{.012} & 28.826\std{.021} & 28.797\std{.010} & 28.723\std{.013} \\
      & ERank (CLIP)         & \textbf{28.874}\std{.015} & \textbf{28.841}\std{.025} & \textbf{28.806}\std{.009} & \textbf{28.743}\std{.013} \\
    \addlinespace
      \midrule
    \multirow{3}{*}{\textbf{Finetune}}
      & Random               & 28.911\std{.005} & 28.879\std{.010} & 28.864\std{.003} & 28.798\std{.006} \\
      & ERank (ResNet-18)    & 28.915\std{.014} & 28.898\std{.005} & 28.882\std{.006} & 28.836\std{.003} \\
      & ERank (CLIP)         & \textbf{28.929}\std{.010} & \textbf{28.913}\std{.004} & \textbf{28.887}\std{.007} & \textbf{28.846}\std{.006} \\
    \midrule
    \multicolumn{6}{@{}l}{\emph{IIIT5K OCR, word accuracy $\uparrow$} - \textbf{remove high-ERank images}}\\
    \midrule
    \multirow{3}{*}{\textbf{Pretrain}}
      & Random               & 0.346\std{.022} & 0.323\std{.055} & \textbf{0.285}\std{.027} & 0.021\std{.010} \\
      & ERank (ResNet-18)    & 0.308\std{.073} & 0.261\std{.076} & 0.116\std{.038} & 0.019\std{.010} \\
      & ERank (CLIP)         & \textbf{0.361}\std{.026} & \textbf{0.340}\std{.042} & 0.224\std{.067} & \textbf{0.027}\std{.043} \\
    \addlinespace
    \midrule
    \multirow{3}{*}{\textbf{Finetune}}
      & Random               & \textbf{0.206}\std{.017} & 0.188\std{.011} & 0.167\std{.006} & 0.126\std{.010} \\
      & ERank (ResNet-18)    & 0.199\std{.007} & 0.182\std{.002} & 0.165\std{.007} & 0.137\std{.001} \\
      & ERank (CLIP)         & 0.202\std{.004} & \textbf{0.189}\std{.008} & \textbf{0.176}\std{.004} & \textbf{0.143}\std{.002} \\
    \bottomrule
  \end{tabular}
\end{table*}

\textbf{Where ERank does not help.} On the remaining tasks, ERank selection
shows no significant improvement over random at the reported budgets. For
classification (MNIST/CIFAR-10, ResNet-18), ERank tracks random closely:
sample usefulness is governed by class structure, to which an unsupervised
richness score is blind. For denoising (DIV2K, EDSR, $\sigma{=}25$), ERank
stays within ${\sim}0.03$~dB of random. For segmentation (Pascal VOC,
DeepLabV3--ResNet50), random is noticeably better at aggressive removal and
on par elsewhere, plausibly because segmentation quality depends on spatial
layout, to which ERank is invariant by construction
(Appendix~\ref{sec:properties}, property~(ii)). In short, ERank helps only
when visual richness aligns with sample usefulness. While constructing more
sophisticated active-learning-style pipelines with scheduled sample
complexity could improve our results, it is beyond the scope of this work.

\vspace{-1em}
\section{Conclusion}
\label{sec:conclusion}

We proposed the effective rank of the channel covariance of an
image's feature map as a per-sample, label-free measure of visual
richness, computable from a single forward pass through an
off-the-shelf encoder. Empirically, ERank orders images from plain to visually rich and
correlates with codec bitrate, sharpness, edge density, and human
complexity annotations on IC9600. It further serves as a
data-selection signal for super-resolution and OCR.

The main practical takeaway is the boundary of the measure. ERank
helps exactly when visual richness is the governing axis of task
difficulty, as in super-resolution, where flat images carry little
signal, and in clutter-hard OCR. It reduces to random selection when
difficulty is set by class labels, by spatial layout, or by noise, as
our classification, segmentation, and denoising experiments show.

\textbf{Limitations.} ERank is relative to a
choice of encoder and layers, our layers were fixed a priori and not
tuned, and results may differ under a systematic layer-selection
study. Our evidence covers two backbones, moderate pruning ratios, and
compact task models such as EDSR and CRNN, so whether the selection
gains persist for stronger, large-scale models in pretraining and
finetuning remains open.

Natural next steps are to use ERank as a differentiable training
objective rather than only a selection criterion, and as a cheap
complexity signal where a dedicated predictor is used today, such as
auxiliary targets and loss weighting~\citep{feng2023ic9600} or
complexity-routed adaptive computation~\citep{williamslekuona2026imagecomplexityawareadaptiveretrieval}.

\newpage

%
%
\bibliographystyle{plainnat}
\bibliography{main}

\appendix

\section*{Appendix}

\section{A Family of Spectral Effective-Dimension Measures}
\label{sec:family}

ERank, the participation ratio, and the stable rank are all members of
a single one-parameter family of spectral measures, and several
statements below are cleanest when read against it. We therefore begin
by defining the family, the propositions that follow are then proved
for whichever member is most tractable and transferred to ERank by the
family's monotonicity.

Let $\lambda_1 \ge \dots \ge \lambda_C \ge 0$ be the eigenvalues of a
covariance (or Gram) matrix, in our case the channel covariance
$\Sigma = \tfrac{1}{N}\bar{X}^\top\bar{X} \in \mathbb{R}^{C\times C}$,
where $\bar{X}$ is the column-centered feature matrix  and let
$p_i = \lambda_i / \sum_j \lambda_j$ be the normalized spectrum, so
that $p = (p_1, \dots, p_C)$ is a probability distribution over the $C$
directions. The \emph{Rényi entropy of order} $\alpha \ge 0$ of this distribution is
\begin{equation}
  \label{eq:renyi-entropy}
  H_\alpha(p) \;=\; \frac{1}{1-\alpha}\,
  \log \sum_{i=1}^{C} p_i^{\alpha},
\end{equation}
with the limiting cases $H_1(p) = -\sum_i p_i \log p_i$ (the Shannon
entropy) and $H_\infty(p) = -\log \max_i p_i$. Exponentiating yields a
family of \emph{effective ranks},
\begin{equation}
  \label{eq:eff-rank-family}
  \operatorname{erank}_\alpha
  \;=\;
  \exp\!\big(H_\alpha(p)\big)
  \;=\;
  \Big(\textstyle\sum_{i=1}^{C} p_i^{\alpha}\Big)^{\!\frac{1}{1-\alpha}},
\end{equation}
each an \emph{effective number} of active directions,
$\operatorname{erank}_\alpha = C$ when the spectrum is flat
($p_i = 1/C$) and $\operatorname{erank}_\alpha = 1$ when it collapses
onto a single direction, for every $\alpha$. In ecology these are the
Hill numbers~\citep{hill1973diversity,jost2006entropy}, and the same
construction underlies the diversity measure
of~\citep{friedman2023vendi}. 
The order $\alpha$ controls how strongly small eigenvalues are
weighted,  larger $\alpha$ suppresses the tail of the spectrum, so
higher-order members are less sensitive to the many small,
often noise-dominated directions.

Below we mainly consider the following three members:
\begin{itemize}
  \item $\alpha = 1$: \textbf{ERank}~\citep{roy2007effective}, the
  exponential Shannon entropy of the spectrum.
  \item $\alpha = 2$: the \textbf{Participation Ratio}
  \begin{equation}
    \label{eq:pr-def}
    \operatorname{PR}(\Sigma)
    \;=\; \frac{1}{\sum_i p_i^2}
    \;=\; \frac{(\operatorname{tr}\Sigma)^2}{\operatorname{tr}(\Sigma^2)},
  \end{equation}
  \item $\alpha \to \infty$: the \textbf{Stable Rank}
  $1 / \max_i p_i = (\operatorname{tr}\Sigma)/\lambda_{\max}$, governed
  by the single largest eigenvalue.
\end{itemize}
 Two properties of the family are
used repeatedly. First, $\operatorname{erank}_\alpha$ is
\emph{nonincreasing} in $\alpha$, giving in particular
\begin{equation}
  \label{eq:renyi}
  \operatorname{erank}(I) \;=\; \operatorname{erank}_1
  \;\ge\; \operatorname{erank}_2 \;=\; \operatorname{PR}(\Sigma).
\end{equation}
Second, every $\operatorname{erank}_\alpha$ rewards a more even
spectrum and the maximum is reached exactly when
all eigenvalues are equal.  

\section{Basic Properties}
\label{sec:properties}

We record the properties that make ERank well-suited as a per-sample
measure - all are elementary consequences of the definition, of the
family in Sec.~\ref{sec:family}, and of majorization
theory~\citep{roy2007effective,marshall2011majorization}.

\emph{(i) Bounds.}
$1 \le \operatorname{erank}(I) \le \operatorname{rank}(\Sigma) \le
\min(N-1, C)$. The lower bound is attained iff the spectrum concentrates
on a single direction, the upper bound $C$ iff the spectrum is flat
($\lambda_1 = \dots = \lambda_C$). The effective-number interpretation
and the relation $\operatorname{erank} \ge \operatorname{PR}$
(Eq.~\eqref{eq:renyi}) follow directly from Sec.~\ref{sec:family}.

\emph{(ii) Invariances.} ERank is invariant to isotropic rescaling
$X \mapsto \alpha X$, to orthogonal transformations of channel space
$X \mapsto X Q$, and to any permutation of spatial positions applied
jointly across channels ($X \mapsto PX$ with $P$ a permutation matrix),
since $\Sigma = \tfrac1N X^\top P^\top P X = \tfrac1N X^\top X$. The
last invariance is deliberate but consequential.

\emph{(iii) Continuity and differentiability.} Unlike matrix rank,
which is discontinuous and requires an arbitrary threshold, ERank is
continuous in $X$ and differentiable wherever the spectrum is simple.
It can therefore serve not only as a diagnostic but, in principle, as a
training objective or regularizer, as in label-free representation
evaluation~\citep{garrido2023rankme}.

\emph{(iv) Concavity and Schur-concavity.} ERank
is Schur-concave in the spectrum: any smoothing of the eigenvalue
distribution at fixed total variance increases it. Moreover, ERank is
log-concave as a function of the normalized eigenvalues. It is thus a
canonical measure of spectral flatness, maximized exactly by the white
spectrum.

\subsection{Connection to Channel Correlation}
\label{sec:corr}

We make precise the sense in which ERank measures channel
decorrelation. The identity is exact for the order-2 member of the
family (the participation ratio, Eq.~\eqref{eq:pr-def}) and transfers
to ERank through the bound of Eq.~\eqref{eq:renyi}.

\begin{proposition}[ERank and mean squared correlation]
\label{prop:corr}
Let $R \in \mathbb{R}^{C \times C}$ be the channel correlation matrix,
i.e.\ the covariance of the per-channel standardized features, and let
$\bar{r}^2 = \tfrac{2}{C(C-1)} \sum_{i<j} r_{ij}^2$ be the mean squared
pairwise correlation between channels. Then
\begin{equation}
  \label{eq:pr-identity}
  \operatorname{PR}(R)
  \;=\;
  \frac{C}{1 + (C-1)\,\bar{r}^2},
\end{equation}
and consequently, by Eq.~\eqref{eq:renyi},
\begin{equation}
  \label{eq:erank-corr-bound}
  \operatorname{erank}(R)
  \;\ge\;
  \frac{C}{1 + (C-1)\,\bar{r}^2}.
\end{equation}
\end{proposition}

\begin{proof}
Since $R$ has unit diagonal, $\operatorname{tr} R = C$ and
\begin{equation*}
  \operatorname{tr}(R^2)
  = \lVert R \rVert_F^2
  = C + 2\!\sum_{i<j} r_{ij}^2
  = C + C(C-1)\,\bar{r}^2 .
\end{equation*}
Hence $\operatorname{PR}(R) = (\operatorname{tr}R)^2 /
\operatorname{tr}(R^2) = C^2 / \big(C + C(C-1)\bar{r}^2\big)$, which is
Eq.~\eqref{eq:pr-identity}. The
inequality~\eqref{eq:erank-corr-bound} follows from
Eq.~\eqref{eq:renyi} applied to $R$.
\end{proof}

\begin{corollary}
\label{cor:corr}
$\operatorname{erank}(R) = C$ iff all channels are pairwise
uncorrelated, and $\operatorname{PR}(R) \to 1$ as
$\lvert r_{ij}\rvert \to 1$ for all pairs. High mean squared channel
correlation forces low ERank, at the explicit rate of
Eq.~\eqref{eq:pr-identity}.
\end{corollary}

Proposition~\ref{prop:corr} is stated for the correlation matrix. For
the centered but unstandardized covariance $\Sigma$ used in
Eq.~\eqref{eq:erank}, ERank measures two things at once: decorrelation
between channels \emph{and} balance of variance across channels (a
diagonal $\Sigma$ with unequal variances also has reduced ERank). We
regard both as components of richness --- an image exciting many
channels weakly and one strongly is genuinely less rich than one
exciting all channels equally --- but the distinction should be kept in
mind when interpreting scores.

\subsection{Behavior under Noise}
\label{sec:noise}

We add noise to the feature tensor and track its effect on the channel
covariance. Let the corrupted features be $\tilde X = X + E$, where
$E \in \mathbb{R}^{N\times C}$ is zero-mean, independent of $X$, and
isotropic across channels with per-channel variance $\sigma^2$, i.e.\
$\tfrac1N\,\mathbb{E}[\bar E^\top \bar E] = \sigma^2 I_C$ (here $\bar E$
is $E$ with each channel centered).

\begin{lemma}[Tensor noise acts additively on the spectrum]
\label{lem:noise-reduction}
Under the assumptions above, the expected corrupted covariance is
\begin{equation}
  \label{eq:noise-cov}
  \mathbb{E}\big[\tilde\Sigma\big]
  \;=\; \Sigma + \sigma^2 I_C ,
  \qquad
  \tilde\Sigma := \tfrac1N \bar{\tilde X}^\top \bar{\tilde X},
\end{equation}
and the empirical $\tilde\Sigma$ concentrates on this value with an
$O(N^{-1/2})$ error. Consequently each eigenvalue of
$\mathbb{E}[\tilde\Sigma]$ is $\lambda_i + \sigma^2$.
\end{lemma}

\begin{proof}
Expanding $\bar{\tilde X} = \bar X + \bar E$,
\begin{equation*}
  \tilde\Sigma
  = \tfrac1N \bar X^\top \bar X
  + \tfrac1N\big(\bar X^\top \bar E + \bar E^\top \bar X\big)
  + \tfrac1N \bar E^\top \bar E .
\end{equation*}
The first term is $\Sigma$. The cross term has zero expectation, since
$E$ is zero-mean and independent of $X$:
$\mathbb{E}[\bar X^\top \bar E] = \bar X^\top \mathbb{E}[\bar E] = 0$.
The last term has expectation $\sigma^2 I_C$ by the isotropy
assumption. This gives Eq.~\eqref{eq:noise-cov}. Both stochastic terms
are averages of $N$ independent contributions, so they fluctuate around
their means at rate $O(N^{-1/2})$; for $N = HW$ in the thousands the
deviation is negligible. Adding $\sigma^2 I_C$ shifts every eigenvalue
by $\sigma^2$, since $\Sigma$ and $\sigma^2 I_C$ are simultaneously
diagonalizable.
\end{proof}

\begin{proposition}[Monotonicity under decorrelated noise]
\label{prop:noise}
Let $\Sigma_\sigma := \Sigma + \sigma^2 I_C$ be the shifted covariance
of Lemma~\ref{lem:noise-reduction}. Then
$\operatorname{erank}(\Sigma_\sigma)$ is nondecreasing in $\sigma^2$,
strictly increasing unless the spectrum of $\Sigma$ is already flat,
and $\operatorname{erank}(\Sigma_\sigma) \to C$ as $\sigma^2 \to \infty$.
\end{proposition}

\begin{proof}
By Lemma~\ref{lem:noise-reduction} the normalized spectrum is
$p_i(\sigma) = (\lambda_i + \sigma^2)/(\operatorname{tr}\Sigma + C\sigma^2)$.
For $\sigma_2^2 > \sigma_1^2$, adding a constant to every eigenvalue
and renormalizing moves mass from larger to smaller entries, so
$p(\sigma_2)$ is majorized by $p(\sigma_1)$; Schur-concavity
(Sec.~\ref{sec:family}) then gives monotonicity, strict unless $p$ is
already uniform. As $\sigma^2 \to \infty$, $p_i(\sigma) \to 1/C$ for
all $i$, whence $\operatorname{erank}(\Sigma_\sigma) \to C$.
\end{proof}

\end{document}